\documentclass[runningheads]{llncs}

\usepackage{url}
\usepackage{courier}
\usepackage{amsmath}
\usepackage{multirow}
\usepackage{listings}
\usepackage{graphicx}

\long\def\BOC#1\EOC{\message{(Commented text )}}
\long\def\BOCC#1\EOCC{\message{(Commented text )}}
\long\def\BOCCC#1\EOCCC{\message{(Commented text )}}

\def\ar{\leftarrow}
\def\beq{\begin{equation}}
\def\eeq#1{\label{#1}\end{equation}}
\def\ba{\begin{array}}
\def\ea{\end{array}}
\def\i#1{\hbox{\it #1\/}}

\def\C{$\mathcal{C}$} 
\def\BC{$\mathcal{BC}$}
\def\fm{\i{FM}}
\def\mo#1{{\bf{#1}}}

\def\ar{\leftarrow}
\def\rar{\rightarrow}

\def\mvis{\!=\!}
\def\false{\hbox{\sc false}}
\def\true{\hbox{\sc true}}

\def\proof{\noindent{\bf Proof}.\hspace{3mm}}
\def\qed{\quad \vrule height7.5pt width4.17pt depth0pt \medskip}

\newtheorem{thm}{Theorem}

\begin{document}

\mainmatter

\title{{\sc Cplus2ASP}: Computing Action Language ${\cal C}$+ \\in Answer Set
  Programming}

\author{Joseph Babb \and Joohyung Lee}
\institute{
School of Computing, Informatics, and Decision Systems Engineering \\
Arizona State University, Tempe, USA \\
\email{\{Joseph.Babb,joolee\}@asu.edu}
}

\maketitle

\begin{abstract}
We present Version 2 of system {\sc Cplus2ASP}, which implements the
definite fragment of action language ${\cal C}$+. Its input language
is fully compatible with the language of the Causal Calculator
Version~2, but the new system is significantly faster thanks to modern
answer set solving techniques. The translation implemented in the
system is a composition of several recent theoretical results. The
system orchestrates a tool chain, consisting of {\sc f2lp},  {\sc
  clingo}, {\sc iclingo}, and {\sc as2transition}.  Under the
incremental execution mode, the system translates a ${\cal C}$+
description into the input language of {\sc iclingo}, exploiting its
incremental grounding mechanism. The correctness of this execution is
justified by the module theorem extended to programs with nested
expressions. In addition, the input language of the system has many
useful features, such as external atoms by means of Lua calls and the
user interactive mode.  The system supports extensible multi-modal
translations for other action languages, such as ${\cal B}$
and ${\cal BC}$, as well.
\end{abstract}

\section{Introduction} \label{sec:intro}

Action language ${\cal C}$+ is a high level language for nonmonotonic
causal theories, which allows us to describe transition systems
succinctly~\cite{giu04}. The definite fragment of
${\cal C}$+ is expressive enough to represent various properties
of actions, and was implemented in Version 2 of the Causal Calculator
({\sc CCalc})\footnote{%
\url{http://www.cs.utexas.edu/users/tag/cc}}. 
The system translates an action description in ${\cal C}$+ into formulas in
propositional logic and calls SAT solvers to compute the models. 
Though {\sc CCalc} is not a highly optimized system, it has been used
to solve several challenging commonsense reasoning problems, including
problems of nontrivial size~\cite{akm04}, to provide a group of robots
with high-level reasoning \cite{caldiran09bridging}, to give executable
specifications of norm-governed computational societies
\cite{art09,desai07representing}, and to automate the analyses of
business processes under authorization constraints \cite{arm09}.

An alternative way to compute the definite fragment of Boolean-valued
$\cal C$+ is to translate it into answer set programs as studied
in~\cite{mcc97c,ferraris12representing}. The system reported
in~\cite{dog01} and system {\sc coala}~\cite{gebser10coala} are
implementations of this method and accept the definite fragment of
${\cal C}$, a predecessor of language ${\cal C}$+. In particular,
{\sc coala} was shown to be effective for several benchmark problems
due to efficiency of ASP solvers. 

However, the input language of {\sc coala} is missing several important
features of~${\cal C}$+, such as multi-valued fluents, defined
fluents, additive fluents, defeasible causal laws, and syntactically
complex formulas. Also, it does not support many useful language
constructs allowed in the input language of {\sc CCalc}, such as
user-defined macros, implicit declarations of sorts, and external
atoms. 

The design aim of system {\sc Cplus2ASP}~\cite{casolary11representing}
is to utilize the efficient ASP solving techniques as in {\sc coala}
while supporting the full features of the input language of~{\sc CCalc}. 
Its design utilizes a standard library with meta-level sorts and
meta-level variables, which yields a simple modular and extensible
method to represent {\sc CCalc} input programs in ASP. 
However, the first version of the system was a prototype
implementation for a proof of concept.

This paper presents Version 2 of {\sc Cplus2ASP}, which is
significantly enhanced in several ways. 

\begin{itemize}
\item  Its input language is fully compatible with the language of
  {\sc CCalc} incorporating the features that were missing in 
  {\sc Cplus2ASP} v1. 

\item  The system supports extensible multi-modal translations for
  different action languages. Currently, in addition to ${\cal C}$+,
  the system supports language ${\cal B}$~\cite{gel98}, and a recently
  proposed language ${\cal BC}$~\cite{lee13action}. Language ${\cal
    BC}$ combines features of languages ${\cal B}$ and ${\cal C}$, and
  allows Prolog-style recursive definitions, which are not allowed in
  $\cal C$+. 

\item The system provides two execution modes: the command line mode
  and the interactive mode. The interactive mode gives a user-friendly
  interface for running various commands.

\item In {\sc CCalc}, external atoms are useful for some deterministic
  computation which is difficult to express directly in ${\cal C}$+.
  For example, they were utilized in~\cite{caldiran09bridging} for a
  loose integration of task planning and motion planning.
  The new version of {\sc Cplus2ASP} supports this feature by
  utilizing Lua call available in the language of {\sc gringo}. 

\item The new system provides an incremental computation of action
  descriptions, which often saves a significant amount of time. 
  Since the translation of action descriptions into answer set
  programs may contain complex formulas, the justification of this
  computation uses the module theorem
  from~\cite{babb12module}, which extends the module theorem
  from~\cite{janhunen09modularity} to first-order formulas
  under the stable model semantics~\cite{ferraris11stable}.
\end{itemize}


In~\cite{casolary11representing}, the translation of a definite
${\cal C}$+ description into the input language of ASP solvers was
explained in multiple steps. A ${\cal C}$+ description is first
turned into a multi-valued causal theory, and then to a Boolean-valued
causal theory by the method described in~\cite{lee05automated}. 
The resulting theory is further turned into logic programs with nested
expressions by the translation in~\cite{ferraris12representing}, and
then the translation in~\cite{lee09system} is applied to turn it
into the input language of {\sc gringo}. 

In Section~\ref{sec:trans}, we explain the translation in a simpler
way by avoiding reference to causal theories but instead by
using a recent proposal of multi-valued propositional formulas
under the stable model semantics~\cite{bartholomew12stable}. 
A ${\cal C}$+ description is turned into multi-valued formulas under
the stable model semantics, which is further turned into propositional
formulas under the stable model semantics \cite{fer05}. The result is
further turned into the input language of {\sc gringo} by the translation
described in~\cite{lee09system}.
Section~\ref{sec:system} introduces system {\sc Cplus2ASP} v2 and the
features of its input language, and Section~\ref{sec:experiments}
compares the system with other similar systems. Our experiments show
that the new system is significantly faster than the others.

\section{From ${\cal C}$+ to ASP}  \label{sec:trans}

\subsection{Review: Multi-Valued Propositional Formulas} \label{ssec:mvpf}


A {\sl (multi-valued propositional) signature} is a set $\sigma$~of
symbols called {\sl constants}, along with a nonempty finite
set~$\i{Dom}(c)$ of symbols, disjoint from $\sigma$, assigned to each
constant~$c$. $\i{Dom}(c)$ is called the {\sl domain} of~$c$.
A {\sl Boolean} constant is one whose domain is the set~${\{\true,
  \false\}}$. 
An {\sl atom} of a signature~$\sigma$ is an expression of the form
${c\mvis v}$ (``the value of~$c$ is~$v$'') where $c \in \sigma$ and $v
\in \i{Dom}(c)$. 
A {\sl (multi-valued propositional) formula} of~$\sigma$ is a
propositional combination of atoms.
We often write $G\ar F$, in a rule form as in logic programs, 
to denote the implication $F\rar G$. 
A finite set of formulas is identified with the conjunction of the
formulas in the set.

A {\sl (multi-valued propositional) interpretation} of~$\sigma$ is a
function that maps every element of~$\sigma$ to an element in its
domain.  An interpretation~$I$ {\sl satisfies} an atom ${c\mvis v}$,
(symbolically, ${I\models c\mvis v}$) if \hbox{$I(c)=v$}.
The satisfaction relation is extended from atoms to arbitrary
formulas according to the usual truth tables for the propositional
connectives. $I$ is a {\sl model} of a formula if $I$ satisfies it. 
We often write an interpretation $I$ with the set of atoms $c\mvis v$
such that $I(c)=v$. 

The {\sl stable} models of a multi-valued propositional formula can be 
defined in terms of a reduct~\cite{bartholomew12stable}. 
Let $F$ be a multi-valued propositional formula of signature $\sigma$,
and let $I$ be a multi-valued propositional interpretation of
$\sigma$. The reduct $F^I$ of a multi-valued propositional formula $F$
relative to a multi-valued propositional interpretation $I$ is the
formula obtained from $F$ by replacing each maximal subformula that is
not satisfied by $I$ with $\bot$. $I$ is a {\sl (multi-valued) stable
  model} of $F$ if $I$ is the unique multi-valued interpretation
of~$\sigma$ that satisfies $F^I$. 

\begin{example} 
Assume $\sigma=\{c\}$, and $\i{Dom}(c)=\{1,2,3\}$. 
Each of the three interpretations
is a model of $c\mvis 1\ar c\mvis 1$, but none of them is stable
because each reduct has no unique model. 
Formula~$c\mvis 1\ar\neg\neg (c\mvis 1)$ has the same models as 
$c\mvis 1\ar c\mvis 1$, but it has one stable model, $\{c\mvis 1\}$:
the reduct of the formula relative to this interpretation is 
$c\mvis 1\ar\neg\bot$, and $\{c\mvis 1\}$ is its unique model. 
Similarly, one can check that $(c\mvis 1\ar\neg\neg (c\mvis 1))\land
(c\mvis 2)$ has only one stable model $\{c\mvis 2\}$,  which
illustrates nonmonotonicity of the semantics.
\end{example}

\subsection{${\cal C}$+ as Multi-valued Propositional Formulas under SM} 

Begin with a multi-valued signature 
partitioned into {\sl fluent} constants and
{\sl action} constants.  The fluent constants are assumed
to be further partitioned into {\sl simple} and {\sl statically determined}.

A {\sl fluent formula} is a formula such that all constants occurring in it
are fluent constants. An {\sl action formula} is a formula that contains at
least one action constant and no fluent constants.

A {\sl static law} is an expression of the form
\beq
{\bf caused}\ F\ {\bf if}\ G
\eeq{static}
where $F$ and $G$ are fluent formulas.
An {\sl action dynamic law} is an expression of the form~(\ref{static})
in which $F$ is an action formula and $G$ is a formula.
A {\sl fluent dynamic law} is an expression of the form
\beq
{\bf caused}\ F\ {\bf if}\ G\ {\bf after}\ H
\eeq{dynamic}
where~$F$ and~$G$ are fluent formulas and $H$ is a formula, provided that~$F$
does not contain statically determined constants.
A {\sl causal law} is a static law, or an action dynamic law, or a fluent
dynamic law. An {\sl action description} is a finite set of causal
laws. 

An action description is called {\sl definite} if $F$ in every causal
law \eqref{static} and \eqref{dynamic} is either an atom or $\bot$.

For any definite action description~$D$ and any nonnegative
integer~$m$, the multi-valued propositional
theory~$\i{cplus2mvpf}(D,m)$ (``\C+ to multi-valued propositional
formulas'') is defined as follows.\footnote{%
The translation can be applied to non-definite \C+ descriptions as
well, but then the semantics does not agree with ${\cal C}$+.}
The signature of~$\i{cplus2mvpf}(D,m)$ consists of the
pairs $i\!:\!c$ such that
\begin{itemize}
\item
$i\in\{0,\dots,m\}$ and~$c$ is a fluent constant of~$D$, or
\item
$i\in\{0,\dots,m-1\}$ and~$c$ is an action constant of~$D$.
\end{itemize}
The domain of $i\!:\!c$ is the same as the domain of~$c$.  Recall that
by $i\!:\!F$ we
denote the result of inserting $i\!:$ in front of every occurrence of every
constant in a formula~$F$, and similarly for a set of formulas.
The rules of~$\i{cplus2mvpf}(D,m)$ are:
\beq
  i\!:\!F\ar \neg\neg (i\!:\!G)
\eeq{tr1}
for every static law~(\ref{static}) in~$D$ and every~$i\in\{0,\dots,m\}$,
and for every action dynamic law~(\ref{static}) in~$D$ and
every~$i\in\{0,\dots,m-1\}$;
\beq
i\!:\!F\ar \neg\neg(i\!:\!G) \wedge (i\!-\!1\!:\!H)
\eeq{tr2f}
for every fluent dynamic law~(\ref{dynamic})
in~$D$ and every~$i\in\{1,\dots,m\}$;
\beq
0\!:\!c\mvis v \ar \neg\neg (0\!:\!c\mvis v)
\eeq{tr3}
for every simple fluent constant~$c$ and every~$v\in\i{Dom}(c)$.

Note how the definition of $\i{cplus2mvpf}(D,m)$ treats simple fluent
constants and statically determined fluent constants in different
ways: rules~(\ref{tr3}) are included only when~$c$ is simple.

The translation of \BC\ into multi-valued propositional formulas is
similar.  Due to lack of space, we refer the reader
to~\cite[Section~9]{lee13action}.

\subsection{Translating Multi-Valued Propositional Formulas to
  Propositional Formulas under SM}

Note that even when we restrict attention to Boolean constants only, 
the stable model semantics for multi-valued propositional formulas
does not coincide with the stable model semantics for propositional
formulas. Syntactically, they are different (one uses expressions of
the form $c=\true$ and $c=\false$; the other uses propositional
atoms). Semantically, the former relies on the uniqueness of 
(Boolean)-functions, while the latter relies on the minimization on
propositional atoms. Nonetheless there is a simple reduction from the
former to the latter. 

Begin with a multi-valued propositional signature~$\sigma$. 
By~$\sigma^{prop}$ we denote the signature consisting of Boolean
constants~$c(v)$ for all constants $c$ in $\sigma$ and all $v\in
\i{Dom}(c)$.  
For any multi-valued propositional formula $F$ of $\sigma$, by
$F^{prop}$ we denote the propositional formula that is obtained from
$F$ by replacing each occurrence of a multi-valued atom ${c\mvis v}$
with $c(v)$. 
For any constant $c$ with $\i{Dom}(c)$, by $\i{UEC}(c)$ we denote
the existence and uniqueness constraints for $c$: \\[-2mm]
\[
   \bot\ar (c(v)\land c(v'))  \\[-1mm]
\]
for all $v,v' \in \i{Dom}(c)$ such that $v\neq v'$, and  \\[-1mm]
\[ 
  \bot\ar\neg \bigvee_{v\in\mathit{Dom}(c)} c(v)\ .   \\[-1mm]
\]
By $\i{UEC}_\sigma$ we denote the conjunction of $\i{UEC}(c)$ for all
$c\in\sigma$.


For any interpretation $I$ of~$\sigma$, by $I^{prop}$ we denote the
interpretation of $\sigma^{prop}$ that is obtained from~$I$ by
defining 
$I^{prop}\models c(v)$ iff $I\models c\mvis v$.

There is a one-to-one correspondence between the stable models of $F$
and the stable models of $F^{prop}$. The following theorem is a
special case of Corollary~1 from~\cite{bartholomew12stable}.

\begin{thm} \label{prop:ct-definite-e}
Let $F$ be a multi-valued propositional formula of a signature
$\sigma$ such that, for every constant~$c$ in~$\sigma$, $\i{Dom}(c)$
has at least two elements. 
{\bf (I)} An interpretation $I$ of~$\sigma$ is a multi-valued stable model
of~$F$ iff $I^{prop}$ is a propositional stable model
of~$F^{prop}\land\i{UEC}_\sigma$.
{\bf (II)} An interpretation $J$ of $\sigma^{prop}$ is a propositional stable
model of~$F^{prop}\land\i{UEC}_\sigma$ iff $J=I^{prop}$ for some
multi-valued stable model $I$ of $F$.
\end{thm}

\subsection{Incremental Computation of ${\cal C}$+} \label{ssec:icplus}

In answer set planning~\cite{lif02}, the length of a plan needs to be
specified. When the length is not known in advance, a plan can be
found by iteratively increasing the possible plan length. 
{\sc Cplus2ASP} Version~1 calls {\sc clingo} for each such
iteration, resulting in redundant computations each time.


Instead, by default, {\sc Cplus2ASP} v2 uses
{\sc iclingo}, which accepts {\em incremental logic programs}.
%
%
Gebser {\sl et al.}~\cite{gebser11reactive} define an
incremental logic program to be a triple 
$\langle B,P[t],Q[t]\rangle$, where $B$ is a disjunctive logic
program, and $P[t]$, $Q[t]$ are incrementally parameterized
disjunctive logic programs. Informally, $B$ is the {\sl base} program
component, which describes static knowledge; $P[t]$ is the {\sl
  cumulative} program component, which contains information regarding
every step $t$ that should be accumulated during execution; $Q[t]$ is
the {\sl volatile query} program component, containing constraints or
information regarding the final step. 
Conceptually, system {\sc iclingo} computes $B\cup P[1]\cup\dots\cup
P[k]\cup Q[k]$ by increasing $k$ one by one, but avoids reproducing
ground rules in each step. Also, previously learned heuristics,
conflicts, or loops are reused without having to recompute them. This
method turns out to be quite effective. The correctness of this
computation assumes that $\langle B,P[t],Q[t]\rangle$ is {\em acyclic}
\cite{babb12module}. 


Below we show that the translation from \C+ described previously 
can be modified to yield an incremental logic program, which is always
acyclic, and thus can be computed by {\sc iclingo}.




For any ${\cal C}$+ description $D$, and any formula $F(t)$ (called a
{\em query}) of the same signature as $\i{cplus2mvpf}(D,t)$, where $t$
is a parameter denoting a nonnegative integer, we define the corresponding
incremental logic program $\langle B,P[t],Q[t]\rangle$ as follows: 
\begin{itemize}
\item 
$B$ consists of 
\begin{itemize}
\item  $0\!:\!\i{UEC}(f)$ for every fluent constant $f$; 
\item  $0\!:\!c(v)\ar\neg\neg (0\!:\!c(v))$ for every simple fluent
  $c$ and every $v\in\i{Dom}(c)$; 
\item  $0\!:\!F^{prop}\ar\neg\neg (0\!:\!G^{prop})$ for every static
  law~\eqref{static} in $D$. 
\end{itemize}
\item $P[t]$ ($t\ge 1$) consists of 
\begin{itemize}
\item  $t\!:\!\i{UEC}(f)$ for every fluent constant $f$; 
\item  $(t\!-\!1)\!:\!\i{UEC}(a)$ for every action constant $a$; 
\item  $t\!:\!F^{prop}\ar\neg\neg (t\!:\!G^{prop})$ for every static
  law~\eqref{static} in $D$; 
\item  $(t\!-\!1)\!:\!F^{prop}\ar\neg\neg ((t\!-\!1)\!:\!G^{prop})$
  for every action dynamic law~\eqref{static} in $D$; 
\item  $t\!:\!F^{prop}\ar\neg\neg 
  (t\!:\!G^{prop})\land((t\!-\!1)\!:\!H^{prop})$ for
  every fluent dynamic law~\eqref{dynamic} in~$D$.
\end{itemize}
\item $Q[t]$ is $\bot\ar\neg (F[t])^{prop}$.
\end{itemize}

Upon receiving this input and a range of nonnegative integers $[{\rm
  min}\dots {\rm max}]$, {\sc iclingo} will find an answer set of the
module  $\mo{R}_k$ with $k={\rm min}, {\rm min}+1,\dots$ until it
finds an answer set, or $k={\rm max}$, whichever comes first. 
In~\cite{babb12module}, module $\mo{R}_k$ is defined from $\langle
B,P[t],Q[t]\rangle$ as follows. 
\begin{align*}
 \mo{P}_0 &\ =\ \fm(B,\emptyset), \\
 \mo{P}_i &\ =\ \mo{P}_{i-1}\ \sqcup\ \fm(P[i], \i{Out}(\mo{P}_{i-1})),
 &(1\le i\le k) \\ 
 \mo{R}_k &\ =\ \mo{P}_k\ \sqcup\ \fm(Q[k], \i{Out}(\mo{P}_k))\ .
\end{align*}
(Due to lack of space, we refer the reader to \cite{babb12module} for
the notations.)

The following theorem states the correctness of incremental execution
in {\sc Cplus2ASP}.
\begin{thm}
For any definite ${\cal C}$+ description $D$, any non-negative integer
$k$, and any formula  $F(k)$ of the same signature as
$\i{cplus2mvpf}(D,k)$, 
an interpretation $I$ is a multi-valued stable model of
$\i{cplus2mvpf}(D,k)\cup\{\bot\ar\neg F(k)\}$ 
iff $I^{prop}$ is a stable model of $\mo{R}_k$.
Conversely, an interpretation $J$ is a stable model of $\mo{R}_k$ 
iff $J = I^{prop}$ for some multi-valued stable model of
$cplus2mvpf(D,k)\cup\{\bot\ar\neg F(k)\}$.
\end{thm}

\proof (Sketch)
We can check that $\langle B,P[t],Q[t]\rangle$ obtained from the ${\cal
  C}$+ description and a query as above is acyclic according to
Definition 12 from~\cite{babb12module}. Then the claim follows from
Proposition~5 from~\cite{babb12module}.
\qed

The translation of \BC\ into an incremental logic program is
similar. 


\section{System \textsc{Cplus2ASP} v2} \label{sec:system}

\vspace{-5mm}

\begin{figure}[ht]
\centering 
\includegraphics[width=12cm,height=5cm]{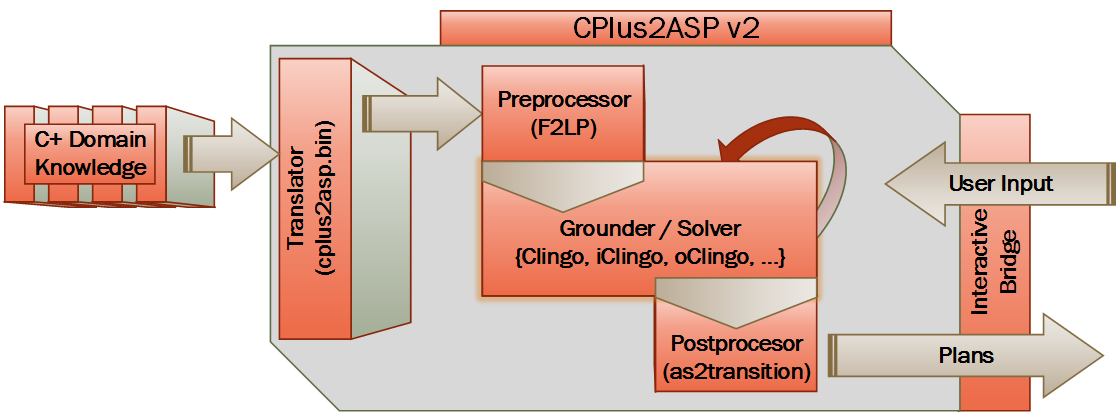}
\caption{{\sc Cplus2ASP} v2 System Architecture} \label{fig:Cplus2ASP2}
\end{figure}

System {\sc Cplus2ASP} v2 is a re-engineering of the prototypical {\sc
Cplus2ASP} v1 system \cite{casolary11representing} and is available
under Version 3 of the GNU Public License. 
\BOCC
{\sc Cplus2ASP} v2 provides
an implementation of the incremental \C+ translation discussed in
Section \ref{sec:prelim} as well as numerous other new features,
including:
\begin{itemize}
\item  Support for all POSIX compliant operating systems,\footnote{%
Support for Windows NT based operating systems is currently under
active development and will be made available in the near future.}

\item  An extensible multi-mode translation and execution system,
  which allows for multiple input and target formalisms.

\item An enhanced, fully interactive, command-line interface,

\item Support for external knowledge via Lua program integration, and
\item Support for the Action Language $\mathcal{BC}$.
\end{itemize}
\EOCC
Like its predecessor, {\sc Cplus2ASP} v2 uses a highly modular
architecture that is designed to take advantage of the existing tools,
including system {\sc f2lp}~\cite{lee09system} and highly-optimized ASP
grounders and solvers in addition to a number of packaged
sub-components.
Figure \ref{fig:Cplus2ASP2} shows a high-level conceptualization of
the interaction of the sub-components in the {\sc Cplus2ASP} v2 
architecture. 

For a description of the input language of {\sc Cplus2ASP},
we refer the reader to the {\sc Cplus2ASP} homepage at 
\url{http://reasoning.eas.asu.edu/cplus2asp} or {\sc CCalc} 2 homepage
at
\url{http://www.cs.utexas.edu/~tag/ccalc/}.
A typical run of {\sc Cplus2ASP} involves the user
interacting with the interactive bridge, a tightly-coupled shell-like
interface for {\sc Cplus2ASP}, in order to configure the {\sc Cplus2ASP}
run. {\sc Cplus2ASP.bin}, a translator sub-component, is then called to compile
a {\sc CCalc} 2 input program into a logic program containing complex
formulas.  Following this, system {\sc f2lp} further turns the
program in the  input language of {\sc gringo}. 
The result of this compilation is given to {\sc clingo}, or a
similar answer set solver, and one or more answer sets are
calculated. Finally, {\sc as2transition} is invoked in order to
format the answer sets into a readable format.


{\sc Cplus2ASP} accepts a {\sc CCalc} 2 style syntax of language ${\cal
  BC}$ as well, for which the user can select a different language
mode for running. In addition, {\sc Cplus2ASP} is able to provide two target
translations, a \textit{static}
translation to traditional ASP, and an \textit{incremental} variant, as
described in section \ref{ssec:icplus}. 


\BOCCC
Figure \ref{fig:bw} shows a formalization of the Blocks World domain
in the input language of {\sc Cplus2ASP}, which is also accepted by
{\sc CCalc}.
For a description of the input language of Cplus2ASP,
we refer the reader to the {\sc Cplus2ASP} homepage at 
\url{http://reasoning.eas.asu.edu/cplus2asp} or {\sc CCalc} 2 homepage
at
\url{http://www.cs.utexas.edu/users/tag/cc}.

\begin{figure}[h!]
\begin{lstlisting}
% File 'bw'

:- sorts
  location >> block.

:- objects
  table                     :: location.

:- constants
  loc(block)                :: inertialFluent(location);
  move(block)               :: exogenousAction;
  destination(block)        :: attribute(location) of move(block).

:- variables
  B,B1                      :: block;
  L                         :: location.

% two blocks can't be on the same block at the same time
constraint B\=B1 & loc(B)=loc(B1) ->> loc(B)=table.  

% effect of moving a block
move(B) causes loc(B)=L if destination(B)=L.

% a block can be moved only when it is clear
nonexecutable move(B) if loc(B1)=B.

% a block can be moved only to a position that is clear
nonexecutable move(B) 
   if destination(B)=loc(B1) & destination(B)\=table.

% a block can't be moved onto a block that is being moved also
nonexecutable move(B) & move(B1) if destination(B)=B1.
\end{lstlisting}
\caption{The Blocks World in {\sc Cplus2ASP}}\label{fig:bw}
\end{figure}

\begin{figure}[h!]
\begin{lstlisting}
% File 'bw-test'

:- include 'bw'.

:- objects
  a,b,c,d                    :: block.

:- query
 label :: simple;
 maxstep :: 2;
 0: loc(a)=b, loc(b)=table, loc(c)=d, loc(d)=table;
 maxstep: loc(a)=table, loc(b)=a, loc(c)=table, loc(d)=c.
\end{lstlisting}
\caption{A Simple Query for The Blocks World in {\sc Cplus2ASP}}\label{fig:bw-test}
\end{figure}

%
%

\NB{For a more complete description of the input language of {\sc
    Cplus2ASP} and the Causal Calculator, we refer the reader to ...}
\EOCCC

\subsection{Running Modes of System \textsc{Cplus2ASP} v2}

In this section we briefly review the usage of \textsc{Cplus2ASP}
v2. For more complete documentation and information on obtaining and
installing \textsc{Cplus2ASP} v2 we invite the reader to visit the
\textsc{Cplus2ASP} homepage.

\textsc{Cplus2ASP} v2 currently offers two distinct user-interaction
methods: command-line and interactive shell. A brief introduction to
both modes is provided below.

\vspace{-2mm}
\subsubsection{Using the Command-Line Mode}

The command-line mode is designed primarily for interacting with a script
or a seasoned {\sc Cplus2ASP} user who is familiar with the options
available to them. The command-line mode is the default
user-interaction mode when a query is provided while calling {\sc
  Cplus2ASP}. 



For instance, to run a query labeled ``{\tt simple}'' on a \C+
description stored in file {\tt bw-test}, one can run the command: 
\begin{verbatim} 
    cplus2asp bw-test query=simple
\end{verbatim} 
In order to run the command under the \BC\ semantics, the flag
\texttt{--language=bc} should be asserted in the command line
call.\footnote{The {\tt bw-test} example program, along with other
  examples, can be found from the {\sc Cplus2ASP} homepage.}

If more solutions are desired, the number of solutions can be appended
to the end of the command-line. As an example, appending
\texttt{4} to the end of the command will return up to four solutions,
while appending \texttt{all} or \texttt{0} will return all solutions.

The system provides the following options to write the output of a
toolchain component into a file. Below \texttt{[PROGRAM]} may be one
of {\tt pre-processor}, {\tt grounder}, 
{\tt solver}, or {\tt post-processor}.

\begin{description}
\item[\texttt{--[PROGRAM]-output=[FILE]}] Writes the output of the
  toolchain component \texttt{[PROGRAM]} to a persistent output file
  \texttt{[FILE]}.

\item[\texttt{--to-[PROGRAM]}] Executes the program toolchain up to
  and including \texttt{[PROGRAM]}. Similarly,
  \texttt{--from-[PROGRAM]} selects a program to initiate execution
  with and continue from.
\end{description}

As an example, if the user wants to run the toolchain up to the
preprocessor and store the results for use later, he could use
the command
\lstset{ 
  numbers=none}
\begin{verbatim}
   cplus2asp bw-test --to-pre-processor > bw-test.lp.
\end{verbatim}
Later, he could then run the command
\begin{verbatim}
   cplus2asp bw-test.lp --from-grounder query=simple
\end{verbatim}
to continue execution.

\BOCC
A partial list of additional command-line options is available below.
\begin{description}
\item[\texttt{--minstep=[VALUE]}] Overrides the query provided minimum
  step value to evaluate (or 0 if no value was provided) for the
  duration of the run. 
\item[\texttt{--maxstep=[VALUE]}] Similar to \texttt{minstep} but for
  the maximum step value to evaluate. This can also be specified in as
  a finite number range, such as \texttt{0..10}.
\item[\texttt{--mode=[MODE]}] By default, {\sc Cplus2ASP} uses the
  \texttt{incremental} translation mode. This option modifies the
  behavior, allowing for the use of a static translation through the
  \texttt{static-auto} or \texttt{static-manual} modes (the
  \texttt{-auto} refers to the level of user interaction;
  \texttt{static-manual} affords the user greater control of the
  maximum step during each increment.).

\item[\texttt{--[PROGRAM]-output=[FILE]}] Writes the output of the
  toolchain component \texttt{[PROGRAM]} to a persistent output file.
  \texttt{[PROGRAM]} may refer to {\tt pre-processor}, {\tt grounder},
    {\tt solver}, or {\tt post-processor}.

\item[\texttt{--to-[PROGRAM]}] Executes the program toolchain up to
  and including \texttt{[PROGRAM]}. Similarly,
  \texttt{--from-[PROGRAM]} selects a program to initiate execution
  with and continue from.

As an example, if we wished to run the toolchain up to the
preprocessor and store the results for use later we could simply use
the command
\lstset{ 
  numbers=none}
\begin{verbatim}
   cplus2asp bw-test --to-pre-processor > bw-test.lp.
\end{verbatim}
Later, we could then run the command
\begin{verbatim}
   cplus2asp bw-test.lp --from-solver query=simple
\end{verbatim}
to continue execution.
\end{description}
\EOCC

\vspace{-2mm}
\subsubsection{Using the Interactive Mode}

The user-interactive mode provides a shell-like interface which allows the
user to perform many of the configurations available from the command
line. In general, the user-interactive mode is entered any time the user 
fails to provide all necessary information within the command-line 
arguments. As such, the easiest way to enter the user-interactive mode is 
to neglect to specify a query on the command-line. As an example, the command
\begin{verbatim}
    cplus2asp bw-test
\end{verbatim}
will enter the user-interactive mode.

While in the user-interactive mode, the following commands, among
others, are available to the user: 
\begin{description}
\item[\texttt{help}] Displays the list of available commands. 
\item[\texttt{config}] Reveals the currently selected running options.
\item[\texttt{queries}] Displays the list of available queries to run. 
\item[\texttt{minstep=[\#]}] Overrides the minimum step to solve for the next query selected.
\item[\texttt{maxstep=[\#]}] Overrides the maximum step to solve for the next query selected.
\item[\texttt{sol=[\#]}] Selects the number of solutions to display.
\item[\texttt{query=[QUERY]}] Runs the selected query and returns the results.
\item[\texttt{exit}] Exits the program.
\end{description}
Following successful execution of a query, the system will return to
the interactive prompt and the process can be repeated. 
For more information on using \textsc{Cplus2ASP} v2, we invite the
reader to explore the documentation available at
\url{http://reasoning.eas.asu.edu/cplus2asp} or within the help usage
message available by executing {\tt cplus2asp --help}.

\subsection{Lua in System \textsc{Cplus2ASP} v2}

System {\sc Cplus2ASP} v2 allows for
embedding external Lua function calls in the system, which are
evaluated at grounding time. These Lua calls allow the user a great
deal of flexibility when designing a program and can be used for
complex computation that is not easily expressible in logic programs. 
A Lua function must be encapsulated in
\texttt{\#begin\_lua \dots \#end\_lua.} tags, and, can
optionally be included in a separate file ending in \texttt{.lua}. Lua
calls occurring within the {\sc Cplus2ASP} program are restricted to
occurring within the \texttt{where} clause \footnote{The condition in
  the \texttt{where} clause is evaluated at 
grounding time.} of each rule and
must be prefaced with an \texttt{@} sign.



For example, one can say moving a block does not always work.\footnote{
Note that this is decided at grounding time so this is not truly
random.}
\lstset{ 
  numbers=none}
\begin{verbatim}
  move(B,L) causes loc(B)=L where @roll(1,2). 
\end{verbatim}
%
with Lua function defined as 
\lstset{ 
  numbers=none}
\begin{verbatim}
  #begin_lua
  math.randomseed(os.time())
  function roll(a,n)--returns 1 with probability a/n
    if(math.random(n) <= a) then return 1
    else return 0
    end
  end
  #end_lua.
\end{verbatim}

A more complete description of the system's Lua functionality and
additional examples of its use are available from the {\sc Cplus2ASP}
homepage.


\section{Experiments}  \label{sec:experiments}

In order to compare the performance of the {\sc Cplus2ASP} v2 system
with its predecessors, we used large variants of several widely known
domains~\footnote{%
All benchmark programs are available from the {\sc Cplus2ASP}
homepage.}
and compared the performance of {\sc Cplus2ASP}'s
running modes with the performance of {\sc CCalc}~v2, {\sc Cplus2ASP}
v1, and the incremental and static running modes of {\sc coala} (where
applicable). All experiments were performed on an Intel Core 2 Duo
3.00 GHZ CPU with 4 GB RAM running Ubuntu 11.10. The \textsc{CCalc} v2
tests used {\sc relsat} 2.0 as a SAT solver while {\sc Cplus2ASP} v1, v2,
and {\sc coala} tests used the same version of {\sc clingo}, v3.0.5. 

The domains tested include a large variant of the Traffic
World~\cite{akm04}, which models the behavior of
cars on a road; a variant of the Blocks World where actions have
costs~\cite{lee03}; the Spacecraft Integer~\cite{lee03}, which models
a spacecraft's movement with multiple
independent jets; the Towers of Hanoi; and the Ferryman domain, which
involves moving a number of wolves and sheep across a river without
allowing the sheep to be eaten.
The Towers of Hanoi and Ferryman descriptions are from examples
packaged with {\sc coala} v1.0.1. In order to run on other systems,
we manually turned them into the syntax of CCalc input language. 


\begin{table}[t]
{\scriptsize 
\begin{minipage}{\textwidth}
\centering
\begin{tabular}{||c|c||c||c||c c||c c||}
\hline
\multirow{2}{*}{\bf Domain	}			
									&	\multirow{2}{*}{\bf steps}			
									& \multirow{2}{*}{\textbf{\sc{CCalc} 2}}	
									& \multirow{2}{*}{\textbf{\sc{Cplus2ASP} v1}}	
									& \multicolumn{2}{|c||}{\textbf{\sc{coala}}}		
									& \multicolumn{2}{|c||}{\textbf{\sc{Cplus2ASP} v2}} \\
									& &	&
									& {static}
									& {incr.}
									& {static} 
									& {incr.}  \\ \hline \hline
								
\BOCC						
\multirow{2}{*}{traffic (jam)}
									&			
									&	\							
									&						
									&
                                                                        \multirow{2}{*}{--}\footnote{--
                                                                          means
                                                                          that
                                                                          the
                                                                          input
                                                                          language
                                                                          of
                                                                          {\sc
                                                                            coala}
                                                                          is
                                                                          not
                                                                          expressive
                                                                          enough
                                                                          to
                                                                          represent
                                                                          the domain.}
                                                                         &
                                                                          \multirow{2}{*}{--}	
									&	&		\\ 
									
									&
									& 
									& 
									& &
									& & \\ \hline
\EOCC

traffic 
									&	\multirow{3}{*}{11}

									& \multirow{2}{*}{878.59 s + 1 s} \footnote{preprocessing time + solving time [\# atoms  / \# rules]}		
									& \multirow{2}{*}{95.43 s + 25.95 s}
									&
                                                                        \multirow{3}{*}{--\footnote{The input language is not expressive enough to represent the domain.}
} 
									& \multirow{3}{*}{--}
									& \multicolumn{1}{|l}{82.16 s}
									& \multicolumn{1}{l||}{14.2 s} \\ 
									
(altmerge)									&
									&
									& 
									& \multirow{3}{*}{--} & \multirow{3}{*}{--}	
									& \multicolumn{1}{|r}{+ 26.57 s}
									& \multicolumn{1}{r||}{+ 2.6 s}  \\
									
									&
									& [531552 / 3671940]
									& [2722247 / 3341068]
									& &
									& \multicolumn{2}{|c||}{[2262231 / 2766459]} \\ \hline \hline
									
bw-cost
									&	\multirow{3}{*}{8}
									
									& \multirow{2}{*}{131.1 s + 5 s}
									& \multirow{2}{*}{76.16 s + 0.4 s}
									& \multirow{3}{*}{--} & \multirow{3}{*}{--}	
									& \multicolumn{1}{|l}{17.09 s}
									& \multicolumn{1}{l||}{3.47 s} \\ 
									
(15)\footnote{maximum cost}									&
									&
									&
									& & 
									& \multicolumn{1}{|r}{+ 3.16 s}
									& \multicolumn{1}{r||}{+ 0.16 s}  \\
									
									& 
									& [149032 / 624439]
									& [123517 / 260282]
									& &
									& \multicolumn{2}{|c||}{[43052 / 526923]} \\ \hline 

bw-cost 
									&	\multirow{3}{*}{9}				
									
									& \multirow{2}{*}{52 s + 987 s}
									& \multirow{2}{*}{271 s + 9.17 s}
									& \multirow{3}{*}{--} & \multirow{3}{*}{--}	
									& \multicolumn{1}{|l}{63.26 s}
									& \multicolumn{1}{l||}{13.45 s} \\ 
									
(20)									&
									&
									&
									& &
									& \multicolumn{1}{|r}{+ 66.58 s}
									& \multicolumn{1}{r||}{+ 2.24 s}  \\
									
									& 
									& [374785 / 1584778]
									& [279869 / 626496]
									&  &
									& \multicolumn{2}{|c||}{[102426 / 1745166]} \\ \hline \hline
								
{spacecraft}
									&	\multirow{3}{*}{3}
									
									& \multirow{2}{*}{173.62 s + 0 s}
									& \multirow{2}{*}{16.07 s + 2.65 s}
									& \multirow{3}{*}{--} & \multirow{3}{*}{--}	
									& \multicolumn{1}{|l}{5.57 s}
									& \multicolumn{1}{l||}{2.33 s} \\ 
									
(15/8)\footnote{domain size ($15 \times 15\times 15$) / goal position ($8\times 8\times 8)$}									&
									&
									&
									& &
									& \multicolumn{1}{|r}{+ 0.06 s}
									& \multicolumn{1}{r||}{+ 0.01 s}  \\
									
									& 
									& [128262 / 622158]
									& [146056 / 146056]
									&  &
									& \multicolumn{2}{|c||}{[132918 / 253514]} \\ \hline
									
spacecraft
									&	\multirow{3}{*}{4}
									&	\multirow{3}{*}{\textit{timeout}}
									
									& \multirow{2}{*}{208.2 s + 480.24 s}
									& \multirow{3}{*}{--} & \multirow{3}{*}{--}
									& \multicolumn{1}{|l}{67.55 s}
									& \multicolumn{1}{l||}{17.46 s} \\ 
									
(25/10)									&
									& 
									& 
									& &
									& \multicolumn{1}{|r}{+ 3.42 s}
									& \multicolumn{1}{r||}{+ 0.35 s}  \\
									
									& 
									& 
									& [760673 / 1653650]
									&  &
									& \multicolumn{2}{|c||}{[732860 / 1427771]} \\ \hline \hline
									
hanoi
									& \multirow{3}{*}{64}
									
									& \multirow{2}{*}{14 s + 1983 s}
									& \multirow{2}{*}{38.9 s + 137.27 s}
									& \multicolumn{1}{|l}{1039.15 s}
									& \multicolumn{1}{l||}{1.4 s}
									& \multicolumn{1}{|l}{547.9 s}
									& \multicolumn{1}{l||}{0.76 s} \\ 
									
(6/3) \footnote{\# disks / \# pegs}									&
									& 
									& 
									& \multicolumn{1}{|r}{+ 507.12 s}
									& \multicolumn{1}{r||}{+ 51.13 s}
									& \multicolumn{1}{|r}{+ 47.53 s}
									& \multicolumn{1}{r||}{+ 3.5 s}  \\
									
									& 
									& [13710 / 221895]
									& [37297 / 298047]
									& \multicolumn{2}{|c||}{[13798 / 410559]}
									& \multicolumn{2}{|c||}{[10086 / 202694]} \\ \hline 
									
towers
									& \multirow{3}{*}{33}				
									& \multirow{3}{*}{\textit{timeout}}
									
									& \multirow{2}{*}{31.19 s + 102.69 s}
									& \multicolumn{1}{|l}{304.02 s} 
									& \multicolumn{1}{l||}{1.51 s}
									& \multicolumn{1}{|l}{102.81 s}
									& \multicolumn{1}{l||}{1.04 s} \\ 
									
(8/4)									&
									&
									& 
									& \multicolumn{1}{|r}{+ 3017.87 s}
									& \multicolumn{1}{r||}{+ 470.23 s}
									& \multicolumn{1}{|r}{+ 89.36 s}
									& \multicolumn{1}{r||}{+ 14.8 s}  \\
									
									& 
									& 
									& [35041 / 433660]
									& \multicolumn{2}{|c||}{[12922 / 655436]} 
									& \multicolumn{2}{|c||}{[9074 / 324668]}  \\ \hline \hline				
{ferryman}
									& \multirow{3}{*}{16}
									
									& \multirow{2}{*}{39.45 s + 0 s}
									& \multirow{2}{*}{8.27 s + 2.98 s}
									& \multicolumn{1}{|l}{40.85 s}
									& \multicolumn{1}{l||}{0.87 s}
									& \multicolumn{1}{|l}{21.59 s}
									& \multicolumn{1}{l||}{0.66 s} \\ 
									
(10/4)	\footnote{\# animals /  boat capacity}								&
									& 
									& 
									& \multicolumn{1}{|r}{+ 8.71 s}
									& \multicolumn{1}{r||}{+ 1.85 s}
									& \multicolumn{1}{|r}{+ 2.37 s}
									& \multicolumn{1}{r||}{+ 0.25 s}  \\
									
									& 
									& [55905 / 308909]
									& [14122 / 120693]
									& \multicolumn{2}{|c||}{{[4973 / 358772]}}
									& \multicolumn{2}{|c||}{{[12721 / 112912]}}  \\ \hline
									
ferryman 
									& \multirow{3}{*}{26}
									
									& \multirow{2}{*}{1004.26 s + 0 s}
									& \multirow{2}{*}{85.21 s + 39.54 s}
									& \multicolumn{1}{|l}{793.13 s}
									& \multicolumn{1}{l||}{6.13 s}
									& \multicolumn{1}{|l}{318.4 s}
									& \multicolumn{1}{l||}{4.18 s} \\ 
									
(15/4)									&
									& 
									& 
									& \multicolumn{1}{|r}{+ 169.18 s}
									& \multicolumn{1}{r||}{+ 14.73 s}
									& \multicolumn{1}{|r}{+ 34.4 s}
									& \multicolumn{1}{r||}{+ 2.97 s}  \\
									
									& 
									& [256590 / 1452554]
									& [42687 / 539513]
									& \multicolumn{2}{|c||}{{[15718 / 2275992]}}
									& \multicolumn{2}{|c||}{{[39536 / 515167]}} \\ \hline \hline	
\end{tabular}
\end{minipage}
}
\caption{Benchmarking Results} 
\label{tbl:benchmarks}
\vspace{-0.7cm}
\end{table}

\begin{figure}[t]
\begin{center}
	\includegraphics[width=12cm,height=4.5cm]{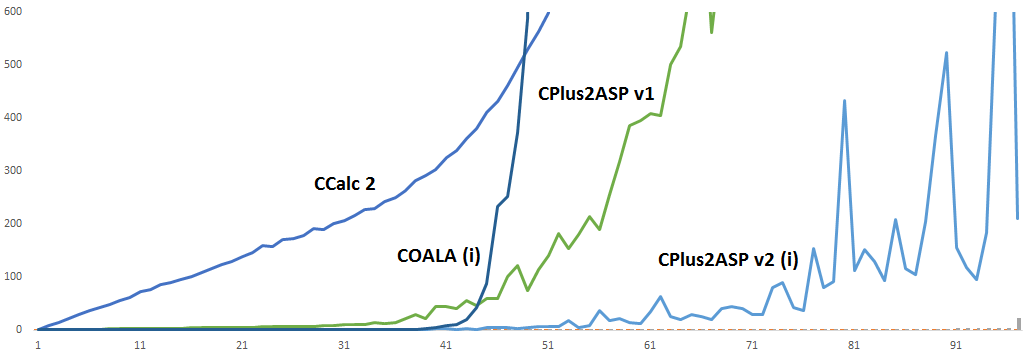}
	\caption{Ferryman 120/4 Long Horizon
          Analysis} \label{fig:long-term} 
\end{center}
\end{figure}

Table~\ref{tbl:benchmarks} compares the results of the test benchmarks
for each of the available configurations. Each measured time includes
translation, grounding, and solving for all possible maximum steps
between 0 and the horizon (\#), as well as the number of atoms and
rules produced below each timing. In all test cases {\sc Cplus2ASP}'s
incremental running mode showed a significant performance advantage
compared to the other systems, performing roughly 3 times faster than
{\sc coala}'s incremental mode and an order of magnitude faster than
its predecessor {\sc Cplus2ASP} v1. {\sc coala}'s incremental running
mode comes in the second place in all but one benchmark.
{\sc Cplus2ASP} v2's static mode tended to outperform its predecessor
on the more computation-heavy domains with additive fluents, but was
subsequently 
outmatched in the others. Finally, {\sc CCalc} 2 and {\sc coala}'s
static mode came in last (with {\sc CCalc} performing slightly worse
in most cases).

Figure {\ref{fig:long-term}} shows a more detailed analysis of the
execution of 
{the first 100 steps of solving an extreme variant of the ferryman 
domain consisting of 120 of each animal} by graphing the time
spent (in seconds) on each step by each configuration. While the
static configurations were required to completely re-ground and
re-solve the translated answer set program for each maximum step,
resulting in an ever-growing amount of work to be performed at each
step, {\sc Cplus2ASP} v2's incremental running mode is able to avoid
this by only grounding the new cumulative ($P[t]$) and volatile
($Q[t]$) components and { leveraging heuristics learned from previous
iterations}. This results in { far less time being required
for checking each increment.}

Although {\sc coala}'s incremental mode uses the same reasoning
engine {\sc iclingo} as {\sc Cplus2ASP} v2's incremental mode,
system {\sc Cplus2ASP} sees a significant overall speed-up over 
{\sc coala}. This is related to a significant reduction in the
number of atoms and rules produced during grounding, which also
accounts for far fewer conflicts and restarts during solving in all
test cases.



\section{Conclusion}  \label{sec:conclusion}

A distinct advantage that \textsc{Cplus2ASP}~v2 has over its
prototypical predecessor is that it was re-engineered in order to
allow for far greater flexibility and extensibility via a multi-modal
execution model. This makes it suitable for use as a base-platform for
future input language implementations, input language extensions, or
target languages/platforms.

The advances in ASP solving techniques account for the efficiency
of system {\sc Cplus2ASP}. We expect that the significant speed-up of
the system demonstrated by {\sc Cplus2ASP} v2, as well as the enhanced
expressivity of the input language, 
will contribute to widening application of action languages in various
domains. 

\medskip\noindent
{\bf Acknowledgements:} 
We are grateful to Michael Bartholomew and the anonymous
referees for their useful comments. This work was partially supported
by the National Science Foundation under Grant IIS-0916116 and by the
South Korea IT R\&D program MKE/KIAT 2010-TD-300404-001.

\bibliographystyle{splncs}

\end{document}